\documentclass[runningheads]{llncs}
\usepackage{makeidx}
\usepackage{graphicx}
\usepackage{amsmath,amssymb,verbatim} 
\usepackage{color}
\usepackage[dvipsnames]{xcolor}
\usepackage{float}

\newtheorem{lem}{Lemma}

\newtheorem{thm}[lem]{Theorem}

\newtheorem{prop}[lem]{Proposition}
\newcommand{\R}{\mathbb{R}}
\newcommand{\ep}{\varepsilon}
\newcommand{\cof}{\mathrm{Color\_Offset}}
\newcommand{\cod}{\mathrm{Color\_difference}}
\newcommand{\mcd}{\mathrm{Max\_Color\_Dif}}
\newcommand{\man}{\mathrm{Min\_Angle}}
\newcommand{\rime}{\mathrm{RIMe}}
\newcommand{\imageheight}{25mm}

\newcommand{\imagerow}[2]{
\includegraphics[height=\imageheight]{Reconstructions/#1/#2_Mesh.jpg}
\includegraphics[height=\imageheight]{Reconstructions/#1/#2_Reconstruction_Rec.jpg}
\includegraphics[height=\imageheight]{Reconstructions/#1/#2_Zoom.jpg}
}

\newcommand{\imagefigure}[1]{
\begin{figure}
\centering

\imagerow{#1}{SLIC100} 
\bigskip

\imagerow{#1}{SLIC100_RIMe} 
\bigskip

\imagerow{#1}{SEEDS100}
\bigskip

\imagerow{#1}{SEEDS100_RIMe} 
\bigskip

\imagerow{#1}{Voronoi}
\bigskip

\imagerow{#1}{CCM} 

\caption{\textbf{Row 1:} pixel-based SEEDS.
\textbf{Row 2:} better SEEDS with RIMe.
\textbf{Row 3:} SLIC.
\textbf{Row 4:} better SLIC with RIMe.
\textbf{Row 5:} Voronoi superpixels.
\textbf{Row 6:} CCM superpixels. 
\textbf{Column 1:} superpixel meshes.
\textbf{Column 2:} reconstructions with average colors.
\textbf{Column 3:} zoomed-in parts of reconstructions by small orange rectangles.}
\label{fig:#1}
\end{figure}
}

\begin{document}
	\pagestyle{headings}
	\mainmatter

	\title{Resolution-independent meshes of superpixels}
	
\author{Vitaliy Kurlin \and
Philip Smith}
\authorrunning{V. Kurlin et al.}
%
\institute{Department of Computer Science, University of Liverpool}

	\maketitle

\begin{abstract}
The over-segmentation into superpixels is an important pre-processing step to smartly compress the input size and speed up higher level tasks. 
A superpixel was traditionally considered as a small cluster of square-based pixels that have similar color intensities and are closely located to each other. 
In this discrete model the boundaries of superpixels often have irregular zigzags consisting of horizontal or vertical edges from a given pixel grid. 
However digital images represent a continuous world, hence the following continuous model in the resolution-independent formulation can be more suitable for the reconstruction problem.
\medskip

Instead of uniting squares in a grid, a resolution-independent superpixel is defined as a polygon that has straight edges with any possible slope at subpixel resolution. 
The harder continuous version of the over-segmentation problem is to split an image into polygons and find a best (say, constant) color of each polygon so that the resulting colored mesh well approximates the given image. 
Such a mesh of polygons can be rendered at any higher resolution with all edges kept straight.
\medskip

We propose a fast conversion of any traditional superpixels into polygons and guarantees that their straight edges do not intersect. 
The meshes based on the superpixels SEEDS (Superpixels Extracted via Energy-Driven Sampling) and SLIC (Simple Linear Iterative Clustering) are compared with past meshes based on the Line Segment Detector.
The experiments on the Berkeley Segmentation Database confirm that the new superpixels have more compact shapes than pixel-based superpixels. 
\end{abstract}

\section{Introduction}
\label{sec:intro}

\subsection{Over-segmentation for low-level vision}
\label{sub:over_segmentation }

The important problem in low-level vision is to quickly detect key structures such as corners and edges where color intensities substantially change.
The over-segmentation problem is to split an image into {\em superpixels}, which are small patches of square-based pixels having similar colors and positions. 
\smallskip

\begin{figure}[h]
\begin{center}
\includegraphics[width=0.48\linewidth]{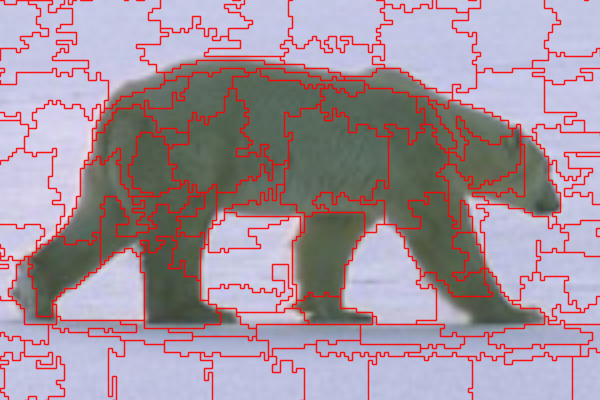}
\,
\includegraphics[width=0.48\linewidth]{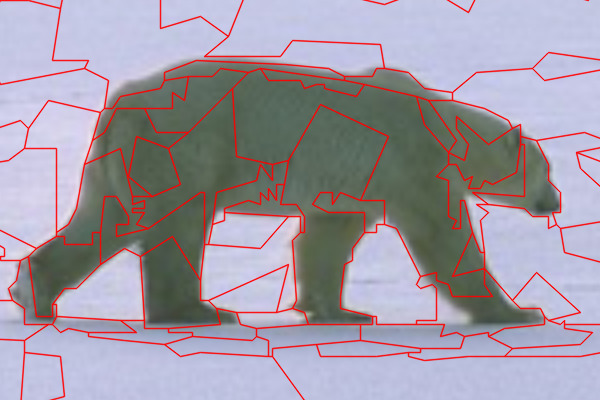}
\end{center}
\caption{SLIC superpixels (left) with zigzag boundaries of pixel-based superpixels are converted into a resolution-independent mesh (right) of polygons with straight edges that can be rendered at any higher resolution for better and smoother animations.}
\label{fig:100007}
\end{figure}

\begin{figure}[h]
\begin{center}
\includegraphics[width=0.48\textwidth]{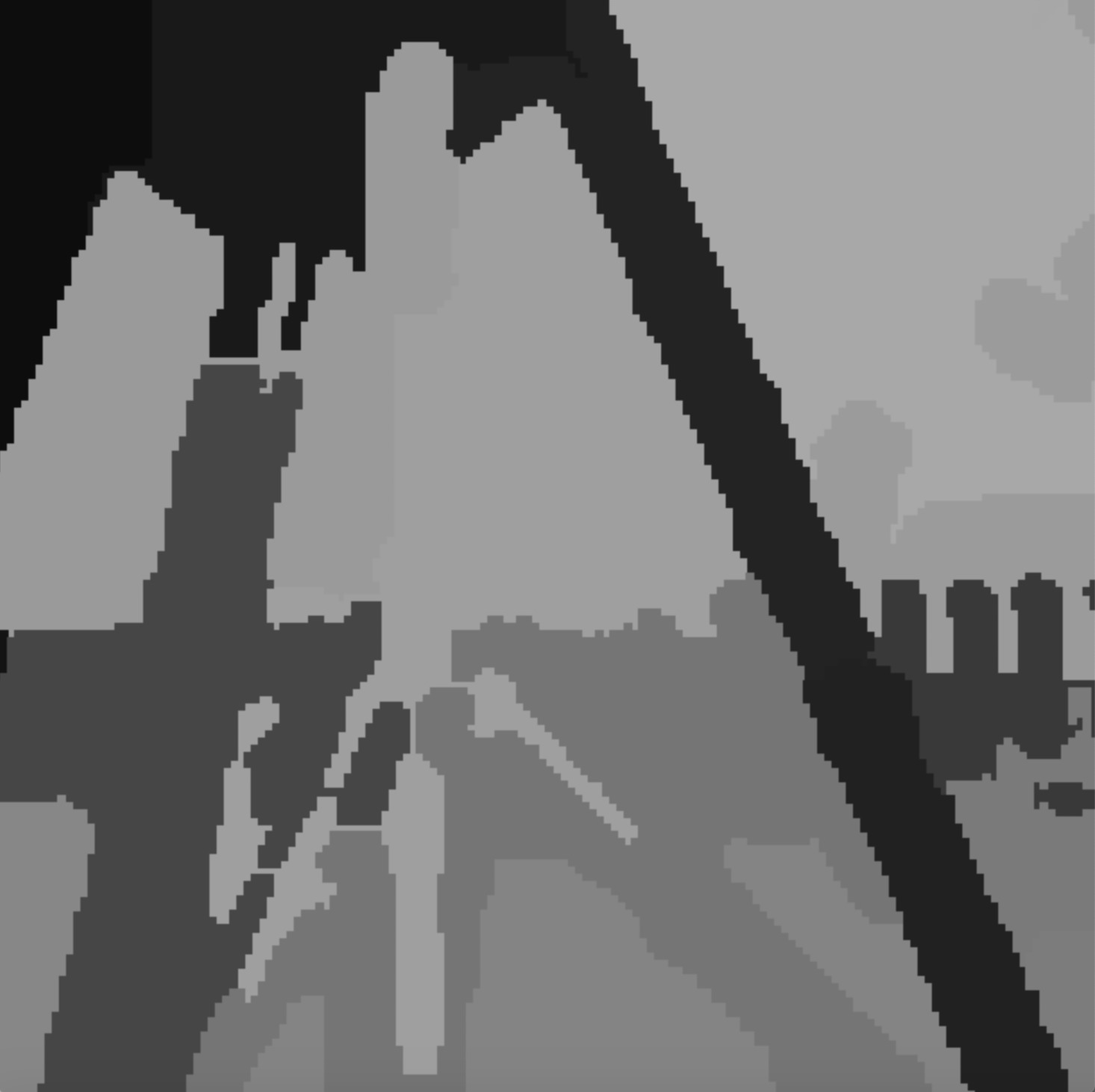}
\includegraphics[width=0.48\textwidth]{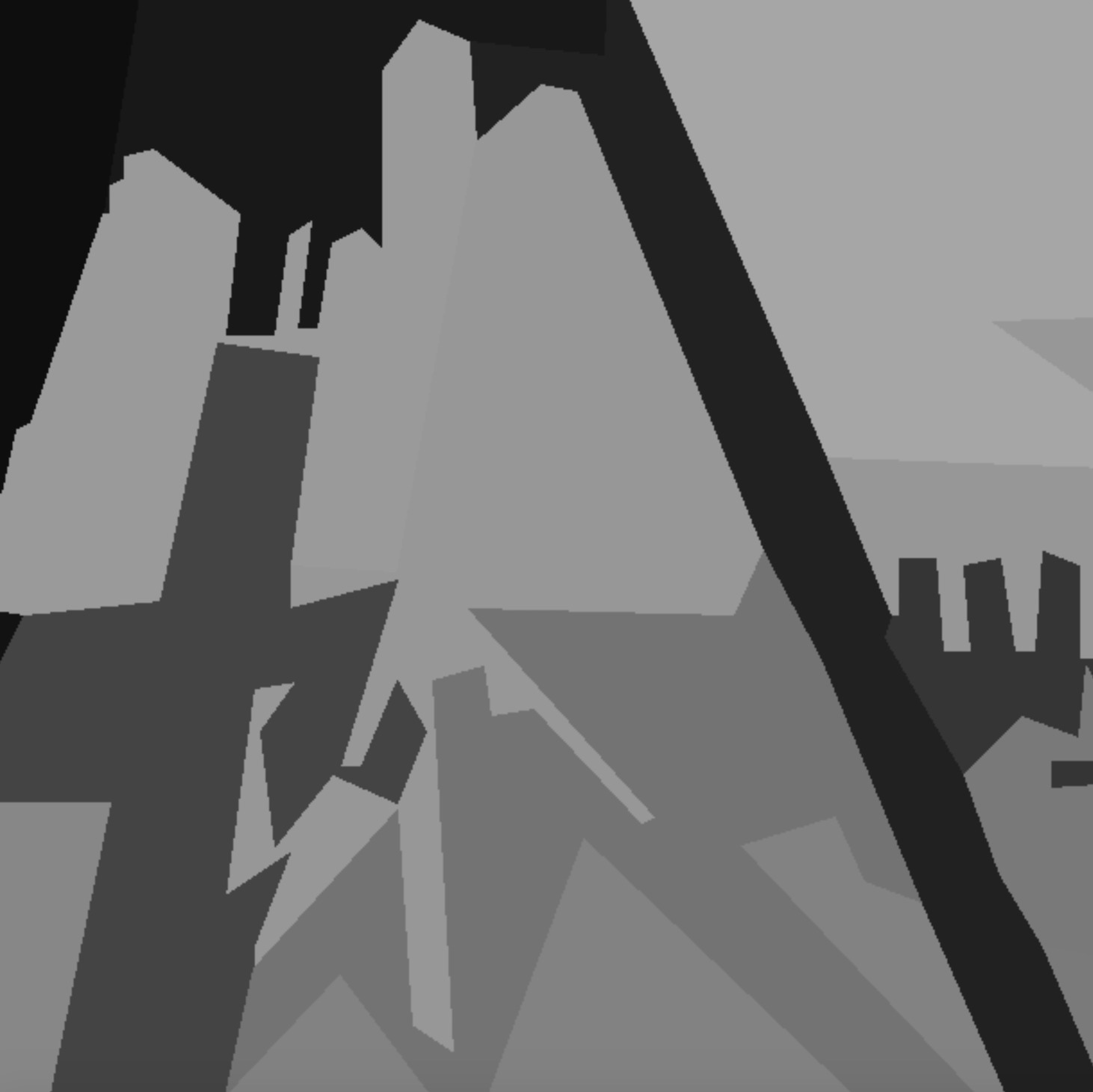}
\includegraphics[width=0.48\textwidth]{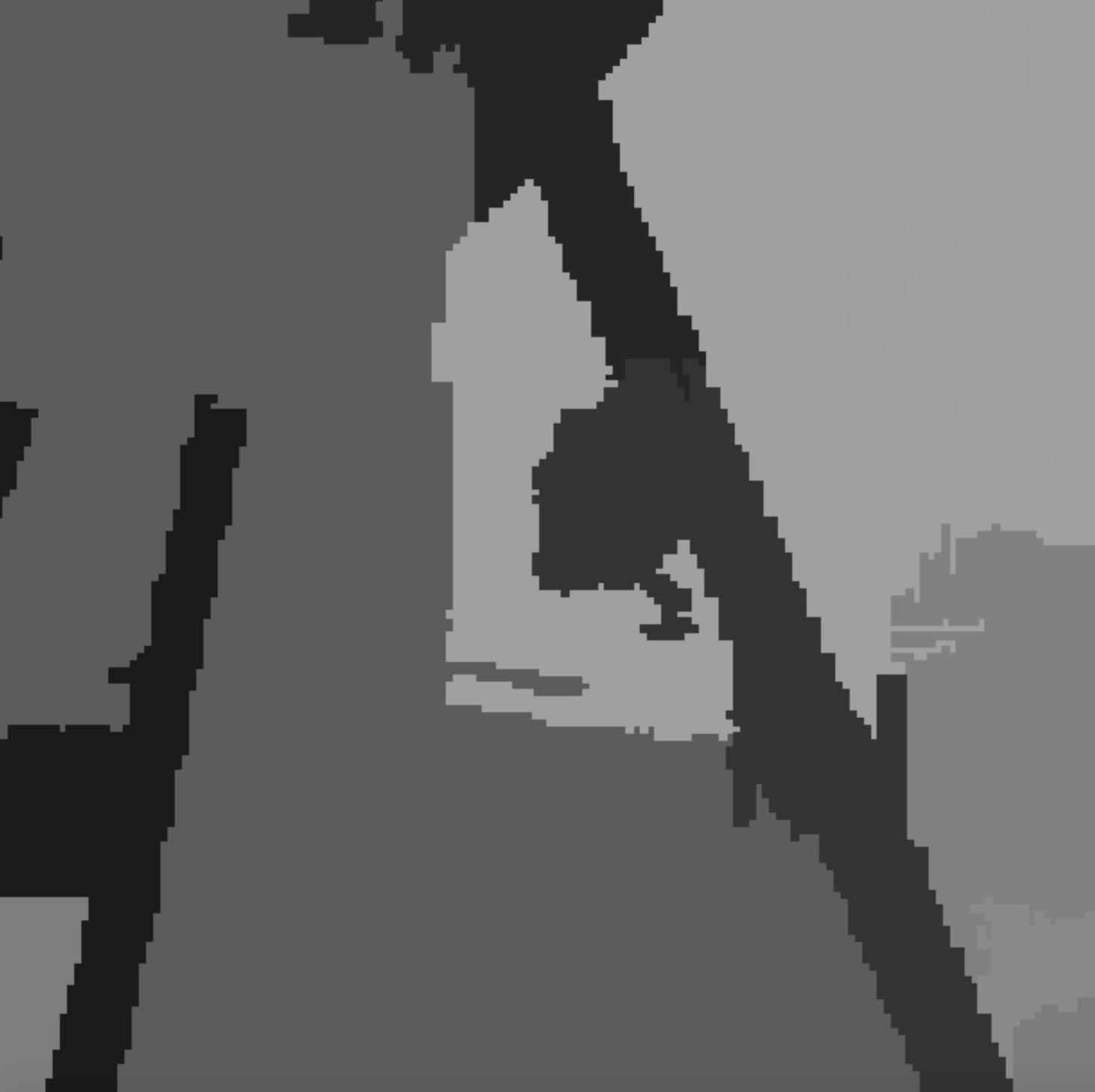}
\includegraphics[width=0.48\textwidth]{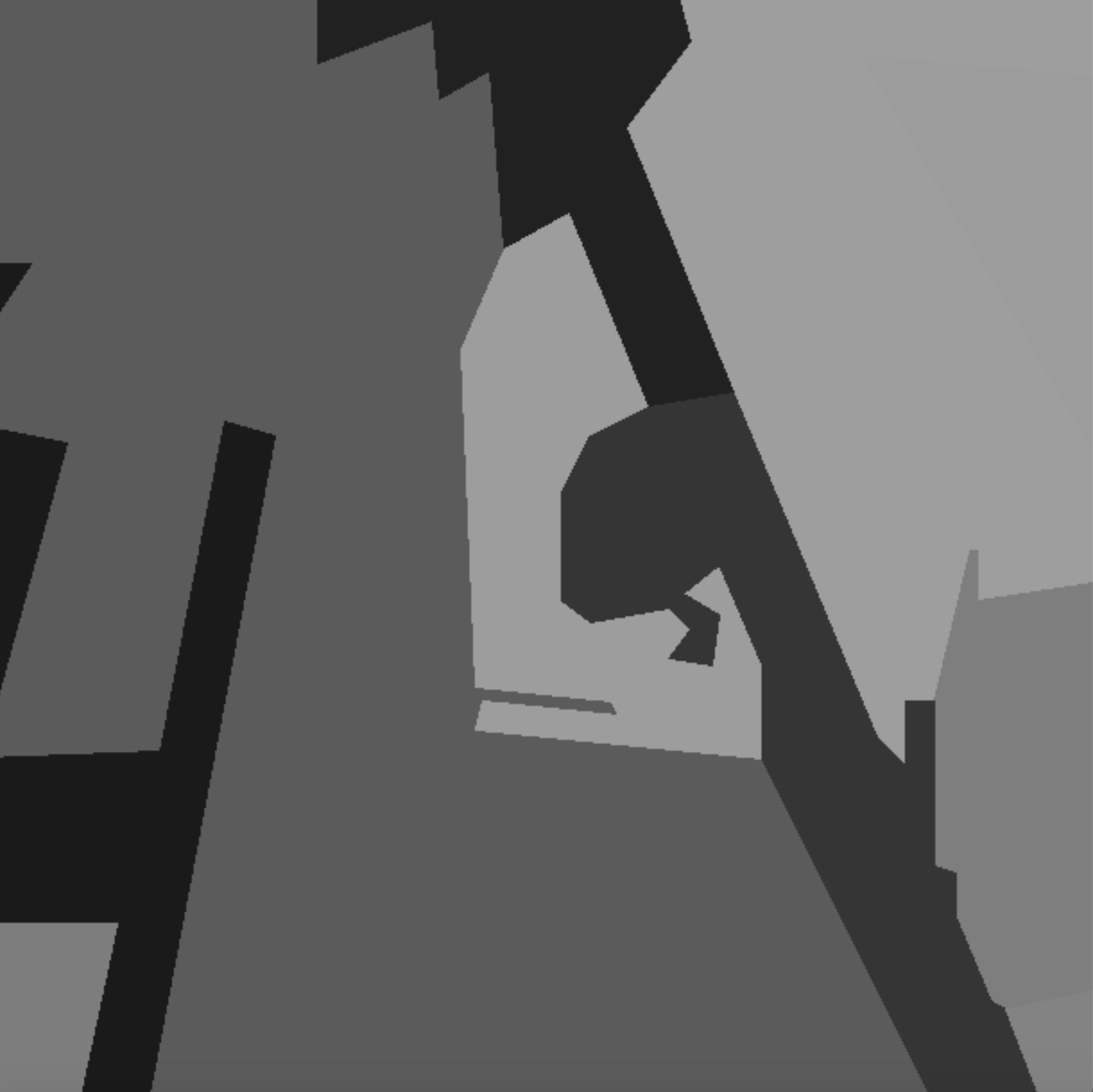}
\end{center}
\caption{The pixel-based superpixels on the left (SEEDS and SLIC) have irregular zigzag boundaries, which are straightened by the algorithm RIMe on the right.}
\label{fig:cameraman}
\end{figure}

Traditional superpixels often have irregular shapes with zigzag boundaries of only horizontal or vertical short edges.
This rigid discretization can be avoided if we allow edges of any length and direction, because continuous objects are much better represented by polygons not restricted to a given pixel grid. 
\smallskip

The color intensity in real images always changes gradually over 2-3 pixels without jumps, see \cite[Fig.~1]{VFC12}.
Hence an edge between different objects can be inside a square pixel, not along its sides.
These hurdles disappear if we look for {\em continuous objects} represented by pixel values discretely sample on a grid.
 
\subsection{Resolution-independent polygonal superpixels}
\label{sub:formulation}

Digital images represent a continuous world around us, but are restricted to a fixed pixel grid.
We consider the over-segmentation problem in the following {\em resolution-independent} formulation introduced by Viola et al. \cite{VFC12}. 
\medskip

\noindent
We split an image into a fixed number of possibly non-convex polygons so that
\smallskip

\noindent
$\bullet$
all polygons have straight edges and vertices with any real coordinates (not restricted to a given pixel grid, so independent of an initial image resolution);
\smallskip

\noindent
$\bullet$
the resulting colored mesh (with a best constant color over each polygon) approximates the original image, e.g. by minimizing an energy in section~\ref{sec:experiments}. 
\medskip

Such a polygonal mesh can be rendered at any higher resolution and is called {\em resolution-independent}.
In general, a {\em mesh} for a rectangular image $I\subset\R^2$ is a graph $G\subset\R^2$ that contains the boundary of $I$ and consists of non-intersecting line segments that split $I$ into possibly non-convex polygons.
Fig.~\ref{fig:cameraman} and~18 show long thin superpixels for the tripod legs in the famous cameraman image.

\subsection{Key contributions to the state-of-the-art for superpixels}
\label{sub:contributions}

\noindent
$\bullet$
We solve the over-segmentation problem for polygonal {\em resolution-independent} superpixels that have few straight edges with infinitely many possible slopes.
\smallskip

\noindent
$\bullet$
The algorithm $\rime$ in section~\ref{sec:algorithm} can convert any pixel-based superpixels into a resolution-independent mesh with quality guarantees in Theorem~\ref{thm:guarantees}.
\smallskip

\noindent
$\bullet$
The experiments in section~\ref{sec:experiments} confirm that the resolution-independent meshes based on SEEDS and SLIC superpixels, achieve better results on objective measures, perform similarly to SEEDS and SLIC on the BSD benchmarks.
\smallskip

\noindent
$\bullet$
$\rime$ beats all other resolution-independent superpixels on the objective reconstruction error and benchmarks of the Berkeley Segmentation Database \cite{BSD}.

\section{Review of the past work on superpixels}
\label{sec:review}

We review the widely used algorithms for pixel-based and resolution-independent superpixels for the harder reconstruction problem of continuous real-life objects.

\subsection{Pixel-based superpixel algorithms}
\label{sub:traditional}

The first successful algorithms were based on the graph of the 4-connected pixel grid  \cite{SM00,FH04,LSKDS09}.
The {\em Lattice Cut} algorithm by Moore et al. \cite{MPW10} guarantees that the final mesh of superpixels is regular as the original grid of pixels.
The best quality in this category is achieved by Entropy Rate Superpixels (ERS) of Lie et al. \cite{LTRC11} minimizing the entropy rate of a random walk on a graph.
Based on {\em Compact Superpixels} by Veksler and Boykov \cite{VBM10}, the fastest algorithm is by Zhang et al. \cite{ZHMB11} processing an average image from BSD500 in 0.5 sec.
Our experiments will use the algorithms SEEDS, SLIC from OpenCV, VLFeat libraries.
\smallskip

The {\em Simple Linear Iterative Clustering} (SLIC) algorithm by Achanta et al. \cite{ASSLFS12} forms superpixels by $k$-means clustering in a 5-dimensional space using 3 colors and 2 coordinates per pixel. 
Because the search is restricted to a neighborhood of a given size, the complexity is $O(kmn)$, where $n$ and $m$ are the numbers of pixels and iterations.
The later  {\em Linear Spectral Clustering} (LSC) by Li et al. \cite{LC15} is based on a weighted $k$-means clustering in a 10-dimensional space. 
The SMURF algorithm by Luengo et al \cite{LBF16} obtains superpixels in a parallelized way within larger super-regions and alternates split/merge steps at several levels.
\smallskip

SEEDS (Superpixels Extracted via Energy-Driven Sampling) by Van den Bergh et al. \cite{BBRG15} seems the first superpixel algorithm to use a {\em coarse-to-fine optimization} that progressively refines superpixels.
At the initial coarse level, each superpixel consists of large rectangular blocks of pixels.
At the next level, all blocks are subdivided into four rectangles and any boundary block can move to an adjacent superpixel, an so on until all blocks become pixels.
SEEDS puts the colors of all pixels within each fixed superpixel are put in bins (five for each color channel) and iteratively maximizes the sum of deviations of all bins from an average bin within every superpixel.
The color deviation is maximal for a superpixel whose pixels have colors in one bin.
\smallskip

The similar Coarse-to-Fine (CtF) algorithm by Yao et al. \cite{YBFU15} minimizes the discrete Reconstruction Error from section~\ref{sec:experiments}.
The recent improvement of this Coarse-to-Fine approach \cite{KH17EMM} allows the user to make shapes of superpixels more round by giving more weight to an isoperimetric quotient. 

\subsection{Edge detection at subpixel resolution}
\label{sub:edge-detection}

Both past algorithms for resolution-independent superpixels (Voronoi and CCM: Convex Constrained Meshes) are based on the Line Segment Detector algorithm (LSDA), which outputs line segments at subpixel resolution \cite{LSD}.
The parameters are a tolerance $\tau$ for angles between gradients and a threshold $\epsilon$ for false alarms.
For the values $\tau=22.5^{\circ}$, $\epsilon=1$ and the random model of a uniformly distributed gradient field,  the LSDA outputs on average at most one false positive.
\smallskip

The LSDA edges have endpoints and gradients with any real coordinates, but the use has no control over a number of line segments in the output.
The recent persistence-based line segment detector \cite{kurlin2019persistence} guarantees edges without intersections and small angles.

\subsection{Resolution-independent polygonal superpixels}
\label{sub:resolution-independent}

The {\em Voronoi superpixels} by Duan and Lafarge \cite{DL15} split an image into polygons not restricted to a pixel grid.
For points $p_1,\dots,p_k$ (called {\em centers}), the {\em Voronoi cell} $V(p_i)$ is the polygonal neighborhood of the center $p_i$ consisting of all points closer (in the Euclidean distance) to $p_i$ than to other centers.
These centers are chosen on both sides of each LSDA edge so that all LSDA edges are covered by the boundaries of Voronoi cells, though no guarantees were proved.
The main advantage of Voronoi superpixels is their almost "round" shape, see section~\ref{sec:experiments}.
\smallskip

The CCM superpixels (Convex Constrained Meshes) by Forsythe et al. \cite{FK17JEI} directly include LSDA edges as hard constraints, which improves the Boundary Recall, see section~\ref{sec:experiments}.
After post-processing LSDA edges to get a straight line graph without self-intersections, this graph is converted into a full mesh of convex polygons that are guaranteed to have no angles smaller than $20^{\circ}$.
The boundaries of CCM superpixels are always in a small offset of LSDA edges.
\smallskip

Both Voronoi and CCM superpixels crucially depend on the quality of LSDA, which may not output a desired number of strongest edges.
Pixel-based superpixels are better optimised for the Boundary Recall.
The goal of the paper is to transfer this advantage to new resolution-independent superpixels.  

\section{RIMe: a resolution-independent mesh of polygons}
\label{sec:algorithm}

\subsection{RIMe algorithm: pipeline, input, parameters and output}
\label{sub:input-output}

The algorithm RIMe converts any pixel-based superpixels into a mesh of polygons whose straight edges approximate boundaries with guarantees in Theorem~\ref{thm:guarantees}.
The {\em input} is the matrix $s(p)$ of superpixel indices for every pixel $p=(i,j)$.

\begin{figure}[h]
	\begin{center}
		\includegraphics[width=\linewidth]{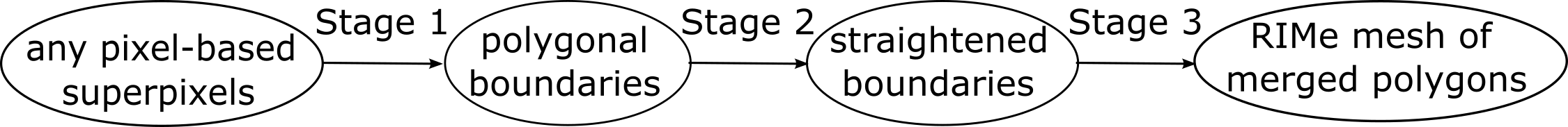}
	\end{center}
	\caption{Pipeline of the RIMe algorithm through 3 stages in subsection~\ref{sub:step-by-step}.}
	\label{fig:pipeline}
\end{figure}

$\rime$ essentially uses the OpenMesh library and saves the final mesh to a .off file.
One can find the most optimal color for every polygon and output the colored mesh as a small representation for further processing, see section~\ref{sec:experiments}. 

\subsection{Conversion algorithm RIMe step-by-step}
\label{sub:step-by-step}

\noindent
{\bf Stage~1}:
Convert any traditional superpixels into a proper pixel-based mesh
by extracting polygonal boundaries of any given pixel-based superpixels.
All past algorithms output a matrix $s(p)$ of superpixel labels (integer indices) over all pixels $p$.
Unfortunately, the union of pixels having the same label $s$ is often disconnected (as a subset in the 4-connected or 8-connected grid).
Even if connected, one superpixel can be surrounded by another superpixel (in this case the surrounding superpixel can be better split into two smaller superpixels).
We convert any output into a mesh whose polygons are connected unions of pixels.
\smallskip

\noindent
{\em Step 1.1}.
Starting from an initial position at a corner of an image, we follow the boundary of a current polygon within the pixel grid and check at every pixel corner whether we should turn left/right or go forward as shown in Fig~\ref{fig:turn-rules}.
\smallskip

\begin{figure}[h]
\begin{center}
\includegraphics[width=\linewidth]{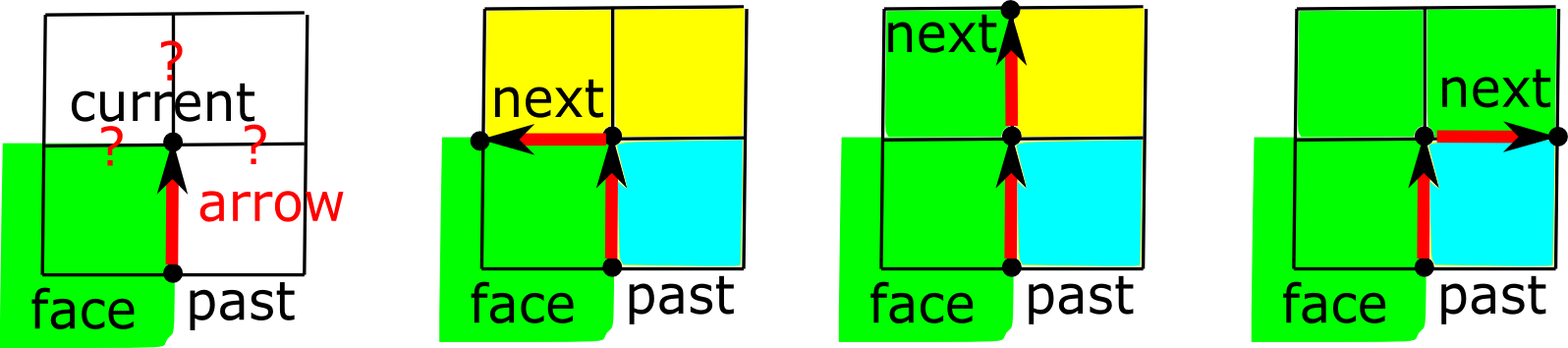}
\end{center}
\caption{Choosing the next point at a corner (colors indicate different superpixels).}
\label{fig:turn-rules}
\end{figure}

\noindent
{\em Step 1.2}.
When we meet a new superpixel, we save an initial arrow (a directed edge between pixels) to make sure that we  later go around the new superpixel.
\smallskip

\noindent
{\em Step 1.3}.
After we have returned to the initial position, a closed boundary of a superpixel was traced and we check if there are any unexplored superpixels.
\smallskip

\noindent
{\em Step 1.4}.
Check that all found superpixels have the expected areas, otherwise we find more connected components by looking at unexplored boundary pixels from the matrix of superpixel labels, so the number of superpixels can increase.
\medskip

\noindent
{\em Step 1.5}.
If a found polygon still does not have the expected area, it must surround another polygon.
We add a straight edge (say, $D_1$) between two closest vertices on the boundaries of the superpixels to split the surrounding superpixel into two.
This edge $D_1$ cannot intersect another edge (say, $D_2$), which would contradict Lemma~\ref{lem:longdiag} for the quadrangle with the intersecting diagonals $D_1, D_2$.
\medskip

\begin{lem}
\label{lem:longdiag}
In any convex quadrilateral with all sides longer than the shortest diagonal, the longest diagonal is longer than any side.
\end{lem}

\begin{proof}
Triangle inequality implies that the sum of the two diagonals is greater than the sum of either pair of opposite sides. 
Assuming that the shortest diagonal is shorter than any side, the longest diagonal should be longer than any side.
\hfill$\square$
\end{proof}

\noindent
{\bf Stage~2}: 
Straighten all boundaries of a mesh from Stage~1.
The key advantage of $\rime$ is the possibility to {\em run Stage~2 in parallel} for different pairs of polygons.
\smallskip

\noindent
{\em Step 2.1}: 
Find the edge chain (a sequence of successive non-boundary edges) between any two adjacent polygons.
\smallskip

\noindent
{\em Step 2.2}: 
Given a chain between vertices $A,B$, find the maximum Euclidean distance $d_C$ from an intermediate vertex $C$ to the straight line $AB$, see Fig.~\ref{fig:straighten-chain}.
\smallskip

\begin{figure}[h]
\begin{center}
\includegraphics[width=\linewidth]{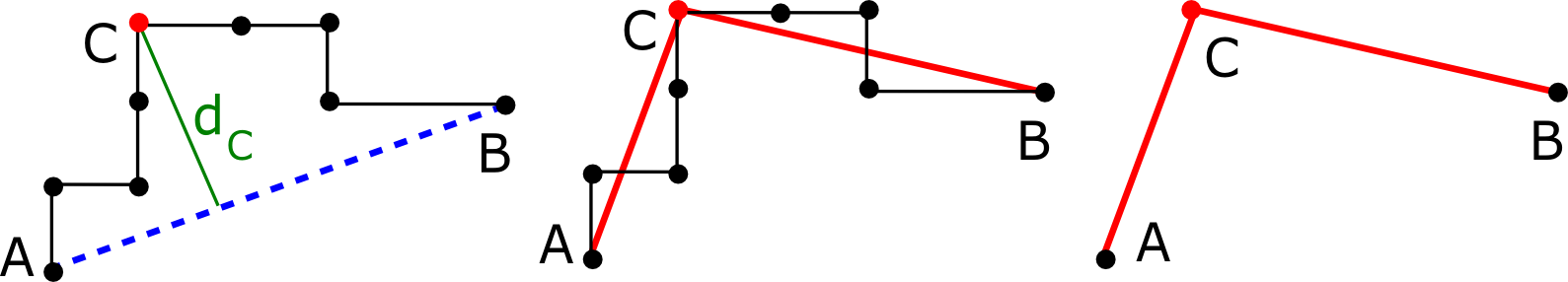}
\end{center}
\caption{Straightening polygonal chains of edges recursively in Step 2.3 when $d_C$ is small.}
\label{fig:straighten-chain}
\end{figure}

\noindent
{\em Step 2.3}: 
We replace the whole chain by the segment $AB$ if
$$\mbox{the max distance }d_C<1\mbox{ or }d_C<\dfrac{\cof}{\cod},\leqno(1)$$
where $\cod$ is the absolute difference between {\em average intensities} (estimated from pixel-based superpixels in Stage 1) of the polygons $F_1,F_2$ on both sides of the chain.
\smallskip

\noindent
{\em Step 2.4}:
For $d_C\geq 1$, we check that the new potential edge $AB$ does not intersect any existing edges of the polygons $F_1,F_2$ sharing the chain. 
We check if $AB$ has angles more than $\man$ all incident edges in the current mesh. 
\smallskip

\noindent
{\em Step 2.5}:
If any condition in Steps~2.3-2.4 fails, we recursively straighten the subchains $AC,CB$ as above. 
\medskip

\noindent
{\bf Stage~3}: 
Merge adjacent polygons whose average intensities differ by less than $\mcd$.
Many images have uniform backgrounds and superpixels consisting of pixels with the same intensity.
Hence merging the resulting polygonal superpixels keeps important edges from a given image.

\subsection{Theoretical guarantees (proved in appendices)}
\label{sub:guarantees}

We assume that pixel-based superpixels are given with their average intensities, otherwise these averages can be quickly computed for Step~2.3.
Apart from these averages, the algorithm $\rime$ accesses only a smaller number $m$ of {\em boundary pixels} that have at least one neighboring pixel from a different superpixel.

\begin{prop}
\label{prop:complexity}
The algorithm $\rime$ in section~\ref{sec:algorithm} has the {\em linear running time} in the number $m$ of boundary pixels that are not strictly inside one superpixel. 
\end{prop}

Let $d$ be the Euclidean distance between points in $\R^2$. 
The {\em $r$-offset} (dilation with a disk of radius $r$) of any $S\subset\R^2$ is its thickened neighborhood $\{p\in\R^2 \;:\; d(p,S)\leq r\}$, where $d(p,S)$ $=\min\{d(p,q) \; : \; q\in S\}$ is the distance from $p$ to $S$.

\begin{thm}
\label{thm:guarantees}
The inequality $d_C<1$ in (1) implies that 
\smallskip

\noindent
(\ref{thm:guarantees}a)
for any pixel-based superpixels, the edges of a resolution-independent mesh can meet only at endpoints.
\smallskip

\noindent
Moreover, the conversion algorithm $\rime$ guarantees that 
\smallskip

\noindent
(\ref{thm:guarantees}b)
all angles between incident edges in a final resolution-independent mesh are at least $\man$,
\smallskip

\noindent
(\ref{thm:guarantees}c)
any chain of edges between original superpixels with $\cod$ is replaced by a polygonal line within the $\dfrac{\cof}{\cod}$-offset of the original chain.
\end{thm}

The value $\cof=30$ means that any chain between polygons whose average intensities differ by 10 is straightened within a 3-pixel neighborhood.

\begin{figure*}[t]
\begin{center}
\includegraphics[height=55mm]{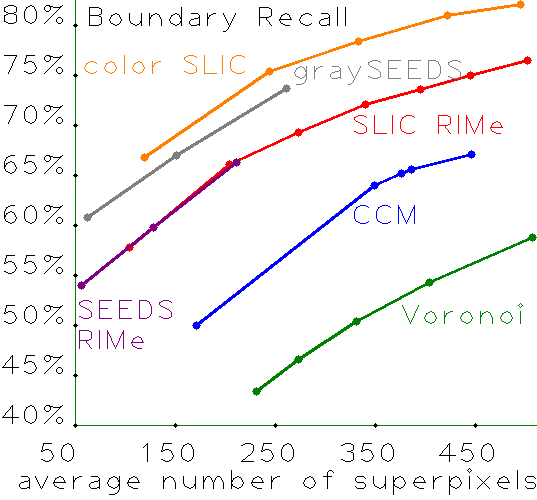}
\hspace*{1pt}
\includegraphics[height=55mm]{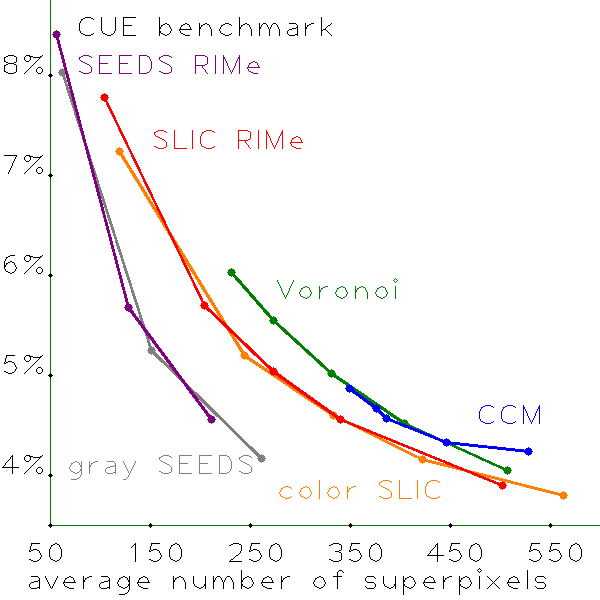}
\end{center}
\caption{Boundary Recall ($\uparrow$ better). Corrected Undersegmentation Error ($\downarrow$ better). }
\label{fig:BR+CUE}
\end{figure*}

\section{Experimental comparison of six algorithms on BSD}
\label{sec:experiments}


The {\em Berkeley Segmentation Database} BSD \cite{BSD} has 500 images widely used for evaluating segmentation algorithms due to (sometimes imperfect) human-sketched ground truth boundaries. 
For an image $I$, let $I=\cup G_j$ be a segmentation into ground truth regions and $I=\cup_{i=1}^k S_i$ be an oversegmentation into superpixels produced by an algorithm.
Each quality measure below compares the superpixels $S_1,\dots,S_k$ with the best suitable ground truth for every image from the BSD.
\smallskip

Let $G(I)=\cup G_j$ be the union of ground truth boundary pixels and $B(I)$ be the set of boundary pixels produced by a superpixel algorithm.
For a distance $\ep$ in pixels, the {\em Boundary Recall} $BR(\ep)$ is the ratio of ground truth boundary pixels $p\in G(I)$ within $2$ pixels from the superpixel boundary $B(I)$.
\smallskip

Van den Bergh et al. \cite{BBRG15} suggested the {\em Corrected Undersegmentation Error} 
$CUE = \dfrac{1}{k}\sum\limits_i  |S_i-G_{max}(S_i)|,$
where $G_{max}(S_i)$ is the ground truth region having the largest overlap with $S_i$.
The {\em Achievable Segmentation Accuracy} is $ASA=\dfrac{1}{k}\sum\limits_i \max\limits_j|S_i\cap G_j|.$
If a superpixel $S_i$ is covered by a ground truth region $G_j$, then $|S_i\cap G_j|=|S_i|$ is the maximum value. 
Otherwise $\max\limits_j|S_i\cap G_j|$ is the maximum area of $S_i$ covered by the most overlapping region $G_j$.
If we use superpixels for the higher level task of semantic segmentation, then $ASA$ is the upper bound on the number of pixels that are wrongly assigned to final semantic regions.
All values of BR, CUE, ASA are in [0,1] and can be measured in percents.

\begin{figure*}
\includegraphics[height=55mm]{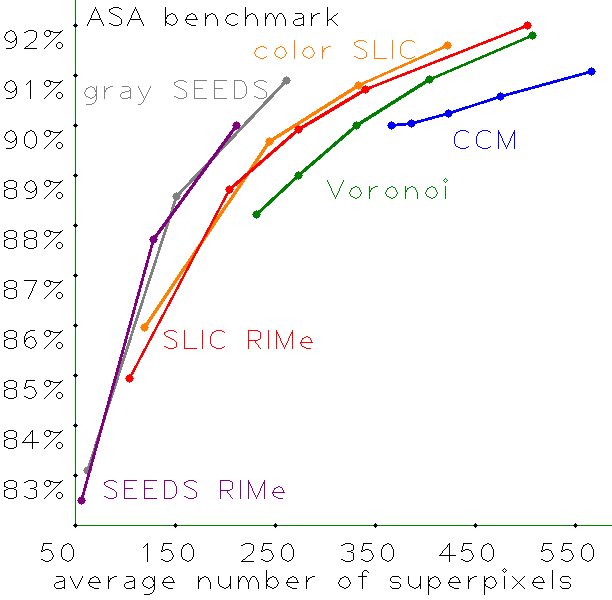}
\hspace*{1pt}
\includegraphics[height=55mm]{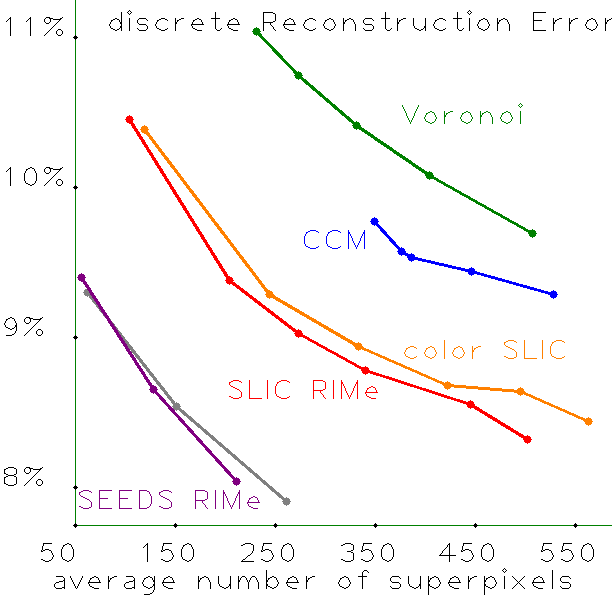}
\caption{
Achievable Segm. Accuracy ($\uparrow$ better).
Discrete Reconstruction Error ($\downarrow$ better).}
\label{fig:ASA+dRE}
\end{figure*}


The BSD benchmarks above involve some parameters, \textit{e.g.} the 2-pixel offset for the Boundary Recall, which can be hard to justify.
That is why several objective cost functions were proposed in an energy minimization framework.
The main energy term for CtF superpixels \cite{YBFU15} depends only on pixel intensities as follows and can be called the {\em discrete Reconstruction Error}:
$$dRE=\sum\limits_{\mbox{pixels } p} 
\Big|\Big|\mbox{Intensity}_p - \mbox{average intensity of }S(p) \Big|\Big|^2,\mbox { where}$$
$S(p)$ is the pixel-based superpixel containing a pixel $p$.
For colored images, the intensity can be considered as a vector of 3 colors with any (say, Euclidean) norm. 
So $dRE$ {\em objectively measures} how well the colored mesh (with average intensities over all superpixels) approximates the original image over all pixels.

\begin{figure}
\begin{center}
\includegraphics[height=55mm]{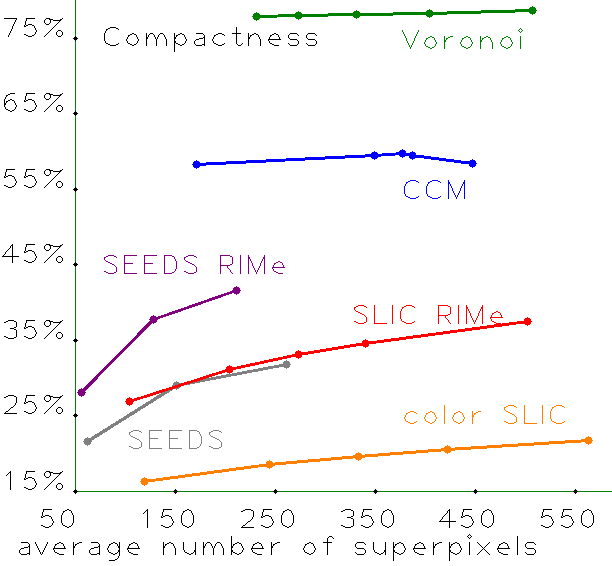}
\hspace*{1pt}
\includegraphics[height=55mm]{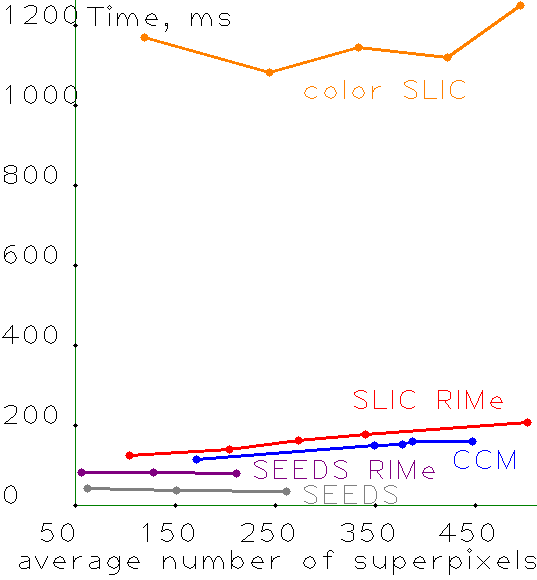}
\end{center}
\caption{Compactness ($\uparrow$ better). Time on 8G RAM 2.6 GHz Intel Core i5 ($\downarrow$ better). }
\label{fig:Comp+Time}
\end{figure}

Compactness measures for superpixels were used as regularizers by different authors \cite{VBM10,YBFU15}.
We propose the simplest version based on the isoperimetric quotient $Q(S)=\dfrac{4\pi\;\mathrm{area}(S)}{\mathrm{perimeter}^2(S)}$, which has the maximum value 1 for a round disk $S$.
The {\em Compactness} is the average $Comp=\sum\limits_{\mbox{superpixels S}}\dfrac{Q(S)}{\#\mbox{superpixels}}$.
\smallskip

Here are the parameters of the RIMe superpixels for benchmarking below. 
\smallskip

\noindent
$\bullet$
$\cof=30$ is used for straightening in Step~2.3 and controls approximation guarantees in Theorem~\ref{thm:guarantees}. 
\smallskip

\noindent
$\bullet$
$\mcd=2$ is the maximum (grayscale) intensity difference for merging adjacent polygons in Stage~3 (larger values will lead to larger superpixels).
\smallskip

\noindent
$\bullet$
$\man=30^{\circ}$ (can be 0) is the minimum angle between adjacent edges (only to avoid narrow triangles).
\medskip

Fig.~\ref{fig:BR+CUE}, \ref{fig:ASA+dRE}, \ref{fig:Comp+Time} show 6 benchmarks for 6 superpixel algorithms.
Each dot on the curves corresponds to a single run on 500 images and has the coordinates (average number of superpixels, average benchmark value over BSD).
\smallskip

The SLIC algorithm crucially uses 3 color values (converted to a Lab space) and we marked its curve as \textcolor{Orange}{color SLIC}.
We ran SEEDS on grayscale version of BSD images and marked the curve by  \textcolor{Gray}{gray SEEDS}, because both Voronoi and CCM algorithms accept only a grayscale input needed for LSDA edges.
The corresponding outputs of $\rime$ are marked by \textcolor{Red}{SLIC RIMe}, \textcolor{Purple}{SEEDS RIMe}.
\smallskip

The two remaining curves are for \textcolor{Green}{Voronoi} \cite{DL15} and \textcolor{Blue}{CCM} meshes \cite{FKF16}.
Fig.~ \ref{fig:223060} shows that RIMe conversions lead to visually better reconstructions than the only other resolution-independent superpixels on Voronoi and CCM meshes.

\section{Conclusions, applications and further problems}
\label{sec:conclusions}

Starting from any pixel-based superpixels, the RIMe conversion produces polygonal superpixels with almost the same BSD benchmarks BR, CUE, ASA and the objective error dRE in Figs.~\ref{fig:BR+CUE}, \ref{fig:ASA+dRE}.
In comparison with any pixel-based superpixels, polygonal superpixels have much fewer edges (no long zigzags in boundaries) with slopes of any potential direction (not only horizontal or vertical).
\smallskip

The RIMe superpixels have twice better compactness (more ``round") shapes than their original pixel-based superpixels such as SEEDS and SLIC in Fig.~\ref{fig:Comp+Time}. 
The RIMe conversions of SLIC, SEEDS outperform other polygonal resolution-independent superpixels (Voronoi and CCM) on BR, CUE, ASA, dRE.
\smallskip

When a resolution of the cameraman image is decreased as in the last figure of supplementary materials, pixel-based superpixels include more and more visible zigzags, while corresponding RIMe superpixels keep nice straight boundaries.
\smallskip

The key advantage of polygonal resolution-independent superpixels is the possibility to render a polygonal mesh at any higher resolution.
This up-scaling can convert low resolution photos into high-resolution images.
In Computer Graphics, polygonal meshes with few edges can be easily manipulated to improve animations converted from traditional videos by cheap cameras.
\smallskip

The next step is to optimize positions of branched vertices, for example by minimizing an energy containing the exact reconstruction error and compactness.

\imagefigure{223060}

{\small
\bibliographystyle{splncs04}
\bibliography{superpixel-meshes}
}

\end{document}